\def\year{2021}\relax
\documentclass[letterpaper]{article} 
\usepackage{aaai21}  
\usepackage{times}  
\usepackage{helvet} 
\usepackage{courier}  
\usepackage[hyphens]{url}  
\usepackage{graphicx} 
\usepackage{natbib}  
\usepackage{caption} 
\usepackage{amsmath,amssymb}
\usepackage{booktabs}
\usepackage{mathtools}
\usepackage{array,tabularx,calc}
\usepackage{amsthm}
\usepackage{bm}
\DeclareMathOperator*{\E}{\mathbf{E}}

\usepackage{xcolor}

\newtheorem{definition}{Definition}
\newtheorem{theorem}{Theorem}
\newtheorem{corollary}{Corollary}

\newtheorem{lemma}{Lemma}

\title{Approximate Novelty Search}
\author{
    \textsuperscript{\rm 1}
    Anubhav Singh, \textsuperscript{\rm 1} 
    Nir Lipovetzky, \textsuperscript{\rm 2} 
    Miquel Ramirez, \textsuperscript{\rm 3} 
     Javier Segovia-Aguas
    \\
}
\affiliations{
    \textsuperscript{\rm 1}
    School of Computing and Information Systems, University of Melbourne,  Australia\\ \textsuperscript{\rm 2}
    Electrical and Electronic Engineering, University of Melbourne,  Australia\\ \textsuperscript{\rm 3}
    Dept. Information and Communication Technologies, Universitat Pompeu Fabra, Spain\\
  
  anubhavs@student.unimelb.edu.au, 
  \{nir.lipovetzky, miquel.ramirez\}@unimelb.edu.au,
  javier.segovia@upf.edu
}
\begin{document}
\maketitle
	\begin{abstract}
		Width-based search algorithms seek plans by prioritizing states according to a suitably defined measure of novelty, that maps states into a set of novelty categories. Space and time complexity to evaluate state novelty is known to be exponential on the cardinality of the set. We present novel methods to obtain polynomial approximations of novelty and width-based search. First, we approximate novelty computation via random sampling and  Bloom filters, reducing the runtime and memory footprint. Second, we approximate the best-first search using an adaptive policy that decides whether to forgo the expansion of nodes in the open list. These two techniques are integrated into existing width-based algorithms, resulting in new planners that perform significantly better than other state-of-the-art planners over benchmarks from the International Planning Competitions.
		
	\end{abstract}
	
	\section{Introduction}
	\label{section:introduction}
	Autonomous systems operating on the edge of computer networks or that only have occasional, 
	sporadic access to vast, centralized computing resources require decision making algorithms
	that work under those conditions. Not only low response times are required to seek 
	courses of action that are safe and effective, but also the memory available for such 
	computation is limited. The need to adapt existing heuristic search algorithms such as
	A* to deal with time and space restrictions was recognized early on~\cite{chakrabati:89}
	and followed-up recently~\cite{vadlamudi:11:mawa,dionne:12:socs}.
	
	In this paper we look at \emph{width-based search methods}~\cite{lipovetzky:12:ecai}, 
	a family of algorithms that rely on heuristics that measure the \emph{novelty} of a 
	state, comparing its information content with that of states visited in the past. 
	Originally developed in the context of classical planning~\cite{geffner:concise},
	when combined with other heuristics~\cite{bonet:aij-hsp,hoffmann:landmarks,lipovetzky:17:aaai,katz:17icaps}, width-based planners become state-of-the-art and competitive with 
	portfolio solvers~\cite{IPC18}.
	A major shortcoming of the latter width-based methods derived from best-first search (BFS)~\cite{edelkamp:heuristics,pearl:heuristics}, such as those used by the planner 
	DUAL-BFWS,
	is that measuring novelty is 
	exponential on the number of discrete levels or categories used to rank states.
	Lipovetzky and Geffner~\shortcite{lipovetzky:12:ecai} showed that
	an upper bound exists for any given classical planning instance, yet this result 
	cannot be exploited, as this
	bound results in impractical runtimes to evaluate states, a crucial issue for the
	effectiveness of heuristic search methods.

	We address this issue by proposing
	new methods to obtain polynomial approximations
	of novelty and to control the growth of the memory footprint of BFWS algorithms. The first contribution is an appraisal approximation of the novelty of state information by
	randomly sampling the space of possible valuations of state variables, and using Bloom filters for efficient, but imprecise, storage of state information. The second contribution is a novel form of best-first search which uses an adaptive policy 
	that decides whether to delay the generation of successor states. This policy is derived from the analytical solution
	to an infinite-horizon Markov Decision Problem (MDP)~\cite{bertsekas:17}, where its cost function
	controls the representation of different novelty categories in the open list.
	
	The paper is structured as follows. We start discussing concisely background material covering classical planning, width-based search, and Bloom filters. Sections~\ref{p-error}
	and~\ref{olc} {expound} 
	the contributions of this paper, approximations of novelty measurements and search. We evaluate both over every benchmark in {the} IPC satisficing track. We finalize with a discussion of the importance and potential impact of our results.
	
	\section{Background}
	\label{section:background}
	
	We follow Bonet \& Geffner~\shortcite{geffner:concise} presentation of classical planning, an
	optimal control problem~\cite{bertsekas:17} where states are fully observable and made of a finite 
	number of Boolean atoms, actions are finite and optimal solutions
	(plans) map the unique initial state into one of many goal states, in the minimum number of steps. 
	The  transition system for a classical planning problem $P$ is defined as $S(P)$ $=$ $\langle S$, $s_0$, $S_G$, $A$,
	$T \rangle$ where $S$ is the state space, $s_0$ $\in$ $S$ is the initial state, $S_G \subseteq S$
	is the set of goal states, $A$ is the set of actions and $T: S \times A \rightarrow S$ is the 
	transition function. We define the set of \emph{applicable} actions $A(s)$ as the subset of
	$a \in A$, for which $T(s,a)$ is defined. A solution to $S(P)$ is a plan $\pi$ $=$ $\langle$ $a_1$, $\ldots$, 
	$a_n \rangle$, with length $n$, that maps an initial state $s_0$ $\in$ $S$ to a goal state $s_n$ $\in$ $S_G$, 
	such that each action is applicable in the corresponding state $a_i$ $\in$ $A(s_{i-1})$, along the 
	induced sequence $s_0$,...,$s_i$,...,$s_n$ where $s_i = T(s_{i-1},a_i)$. 
	
	We describe classical planning problems compactly using
	\textsc{Strips}~\cite{fikes:strips}. $P$ is given as the tuple $P = \langle F, O, I, G\rangle$, where 
	$F$ is the collection of Boolean atoms or fluents, $O$ is the set of operators, $I \subseteq F$ is a set of atoms that fully describe the initial state, and $G \subseteq F$ is a set of atoms in the goal state.
	The transition system $S(P)$ is obtained from $P$  as $S=2^{|F|}$, $s_0=I$,
	$s_G = \{ s \,|\, s \models G\}$, $T$ and $A(s)$ follow from the effects and preconditions of
	$a \in O$. Optimal solutions are those sequences of actions that
	are the shortest too. In this paper we consider \emph{satisficing} solutions, trading off between run time and quality. 
	
	The planners discussed in this paper are instances of the classic heuristic search algorithm named Best-First Search~\cite{edelkamp:heuristics,pearl:heuristics} or BFS for short. This algorithm searches for plans by extending incrementally all paths (nodes) in $S(P)$ starting from $s_0$. The nodes are visited in the order specified by an \emph{evaluation function} $f$ defined over paths, and the algorithm terminates in the first path that ends on a state $s \in S_G$. BFS implicitly \emph{enumerates}
	 the states $s$ in $S(P)$ by assigning to them a natural number, the \emph{expansion order}, $e(s)$~\cite{dionne:12:socs}, which is $0$ for $s_0$ and increases by one unit for every new state $s$ extending an existing path. We denote the set of states generated before $s$ as ${\cal P}(s)$ $=$ $\{\ s'\,|\, e(s') < e(s)\}$. As we will see in the next section, this ordering is crucial to define \emph{novelty measures}.
	
	\subsection{Best-First Width Search (BFWS)}
	\label{section:bfws}
	
	BFWS~\cite{lipovetzky:17:aaai} is a family of BFS algorithms where the evaluation function
	for a node $n$, $f(n)$ is defined as a tuple of functions

		\[
	f(n) = (w, h_1, \ldots, h_m)
	\]

	\noindent where $w: S \rightarrow {\cal W}$ is the function measuring novelty, that maps
	states $s \in S$ into categories $\omega \in {\cal W}$, ${\cal W} \subset \mathbb{N}$, and $H$ $=$ $\{h_1,\ldots,h_m\}$ is a set of suitably
	chosen functions. When inserting nodes $n$ in the open list, BFWS algorithms sort { in increasing order}
	according to the first function in $f(n)$, breaking ties recursively with the provided 
	$h_i$. These functions can also be used to partition the set of states generated before $s$ as $\bm{{\cal P}(s, H)}$ $=$ $\{\ s'\,|\, e(s') < e(s), h(s) = h(s') ~\forall h \in H\}$.

	\begin{definition} 
		\label{def:novelty}
		The \emph{novelty} $w(s)$ $=$ $w_{\langle H \rangle}(s)$ of a newly generated state $s$ given a set of
		partition functions $H$ over states $s \in S$ is $k$, iff (1) exists a 
		tuple\footnote{Conjunction of atoms.} $t \subseteq F$ of \emph{minimum} size $k$, s.t. $s \models t$,
		(2) $\forall{s' \in {\cal P}(s,H)}$, $s' \not\models t$.
	\end{definition}
	As noted in the introduction, BFWS algorithms are state-of-the-art over the IPC classical
	planning benchmarks. We illustrate the BFWS framework by discussing in detail BFWS($f_5$), one of  the best performing to-date in the Agile track (IPC 2018). 

	The evaluation function $f_5$ $=$ $\langle w,\#g\rangle$ makes BFWS to expand first novel states, breaking ties with a simple goal
	counting heuristic~\cite{fikes:strips} $\#g(s)$. The novelty function uses two heuristic functions to partition the novelty space $w = w_{\langle\#g,\#r\rangle}$, one is $\#g(s)$, and the other, $\#r(s)$, 
	counts the atoms $p$ achieved along the path to $s$, 
	such that $p \in R$, $R \subseteq F$, where $R$ is selected by a 
	\emph{relevance analysis} procedure. $R$ is meant to contain atoms which are
	\emph{instrumental} to reach the goal efficiently, so for domain-independent planning, one can
	instance $R$ as a set of landmarks~\cite{hoffmann:landmarks}, or the set of fluents which
	belong to positive effects of actions in the relaxed plan for $s_0$~\cite{hoffmann:ff}. Both of the above definitions of $R$
	were used in the planner DUAL-BFWS, and the later was used in BFWS($f_5$).
	
	Evaluating $w(s)$ requires to test states $s$ to belong to the categories $\omega \in {\cal W}$. Lipovetzky and
	Geffner~\shortcite{lipovetzky:17:aaai} define $\bm{{\cal W}}$ as the integer interval
	$[1, i+1]$ where $i$ is the size of the largest novel tuple generated by optimal plans for $P$.
	In this case, the test above requires to generate exhaustively all tuples $t$ of size $i$
	present in state $s$ and determine if they are present in ${\cal P}(s, H)$.
	An optimal procedure to implement Definition~\ref{def:novelty} follows. For each tuple of size $1$ $\leq$ $l$ $\leq$ $i$,
	let $\bm{\beta_l(s)}$ $=$ $\{ t\,|$ $t \subseteq F$, $s \models t$, $|t| = l \}$\footnote{We note that $\beta_l$ can be iterated
		by lazily generating its elements.}. Starting with $l=1$, we enumerate tuples $t$ $\in$ $\beta_l(s)$, and then test
	if $t$ is part of previously observed tuples in the set $\bm{{\cal N}(s)}$ = $\bigcup_{s'\in {\cal P}(s, H)} \beta_l(s')$.

	If the test is negative for at least one tuple, then
	$w(s) := l$, otherwise, we need to test the elements of $\beta_{l+1}$, until $l=i$. If the test is positive
	for all $t$ $\in$ $\beta_{i}$, then the state is considered not to be novel, and $w(s)$ $:=$ $i+1$.
	This leads to the exponential time and space $O(|F|^{i})$ requirements to evaluate $w(s)$ that we outlined
	in Section~\ref{section:introduction}. While $i$ is known to be generally way smaller than $|F|$, the bound $i$
	is usually high enough to render the novelty test up to $i$ impractical. In the case of
	the IPC benchmarks, any value of $i > 2$ leads to very high runtimes to generate states. 
	In practice, $w(s)$ is approximated by setting $i$ to an arbitrary lower bound, which renders 
	the evaluation of $w(s)$ to be tractable but may relegate states with valuable information
	to the back of the open list.
	
	\subsection{Bloom Filters}
	\label{section:bloom_filters}
	The \emph{Bloom filters}~\cite{bloom,louridas:algorithms} are a probabilistic data
	structure to represent sets efficiently, at the expense of allowing false positives when
	testing whether the set contains a given object. Typical implementations of Bloom filters
	consist of a bit-array $\bm{v}$ of size $\bm{r}$, where all entries $v_j$ are initially set to $\bot$, and
	$\bm{K}$ independent hash functions $\bm{\eta}$ that map objects into the range $[1, r]$. To add an object $o$
	as a member of the set represented by the Bloom filter, the $K$ hash functions $\eta_l$ are evaluated
	on $o$, so $v_{\eta_{l}(o)} := \top$ for $l=1$,..., $K$. To test whether $o$ is in the set,
	the hash functions $\eta_l$ are evaluated, and if all $v_{\eta_{l}(o)} := \top$, then $o$ is considered
	to be an element of the set. In comparison with a traditional hash table whose size grows with
	that of the range of possible objects, the Bloom filter has fixed-size $r$.
	
	The choice of value for $r$ determines the probability to obtain a \emph{false positive}, that is,
	testing $o$ for containment and getting a positive answer, when $o$ has not been previously added as
	a member. As  noted in \cite{Andrei-Michael} the probability of a false positive is given by
	\begin{align}
	P_{f} = \left( 1- e^{-\frac{Kq}{r}} \right)^{K} \label{eq:bloom_filter}
	\end{align}
\noindent where $q$ is the \emph{expected} number of different objects to be tested. The analytical solution to the problem of minimization of false positive rate with respect to $K$ shows that $P_f$ is minimized when $K$~=~$(r/q)~\ln{2}$ \cite{Andrei-Michael}.  {Since the expected number of different objects, nodes in planning, tend to be larger than the memory}, it follows that {when} $(r~\ln{2}) \leq q$,  then $K$~=~$1$ minimizes $P_f$. 

	\section{Novelty Approximation}
	\label{p-error}
	In this section, we describe an approximate measure of novelty for newly generated states, $\hat{w}(s)$, 
	which is tractable and can be proved to be equal to $w(s)$ with positive probability. For that, 
	Definition~\ref{def:novelty} is relaxed as follows
	\begin{definition} 
		\label{def:approx_novelty}
		The \emph{approximate novelty} $\hat{w}(s)$ $=$ $\hat{w}_{\langle H \rangle}(s)$ of 
		a newly generated state $s$ given a set of
		partition functions H over states $s \in S$ is $k$, iff (1) exists a 
		tuple $t \subseteq F$ of \emph{minimum} size $k$, s.t. $t \in Z_k(s)$, (2) 
		$\forall{s' \in {\cal P}(s, H)}$, ${\cal O}(s', t)=\bot$. 
		
	\end{definition}
	\noindent We have changed Definition~\ref{def:novelty} in two ways. First, we only test the
	tuples from a randomly sampled set $\bm{Z_{l}(s)}$ $\subseteq$ $\beta_l(s)$, 
	for $l=1$,...,$i$. 
	We require that the probability of every tuple in $\beta_l(s)$ being selected is  uniformly
	distributed. For that we sample \textit{without replacement} from the \emph{discrete uniform distribution} over $\beta_l(s)$, with probability mass function $p_{X}(t)$ = $z/|\beta_{l}(s)|$ for all $t\in \beta_{l}(s)$, $z$ = $|Z_l(s)|$, representing the probability of occurrence of $t$ in $Z_l(s)$. We note that tuples are sampled, independently, for each state.
	Second, we replace the condition $s' \not\models t$ for a random variable $\bm{{\cal O}(s', t)}$ that models the runtime
	behavior of a Bloom filter. ${\cal O}(s',t)$ maps pairs $s', t$ to $\top$ with probability $1$ when $s' \models t$ and $t \in Z_l(s')$, otherwise, it maps $s',t$ to $\bot$ with positive probability. 
	The tractability of $\hat{w}(s)$ follows
	from requiring $r$, the number of entries in
	the Bloom filter, and $|Z_l(s)| \leq \bar{Z}$, where $\bm{\bar{Z}}$ is the maximum size of $Z_l(s)$, to be constants, e.g. $ \bar{Z}$ $=$ $r$ $=$ $|F|$. It is trivial to note that the running time of any reasonable algorithm for {computing $\hat{w}_{\langle H \rangle}(s)$ as per}
	Definition~\ref{def:approx_novelty} is $O(i\bar{Z})$ and memory requirement is $O(ir)$, dropping the complexity of tuple membership checks from exponential to linear on $i$.
	
	These two simple changes suffice to allow measures of novelty that are much finer than what can be 
	obtained with highly optimized implementations of $w(s)$, but certainly, and as noted at the beginning of
	the section, there is a certain probability that $\hat{w}(s)$ and $w(s)$ will not be in agreement. The rest
	of this section is devoted to provide a probabilistic model of the rate at which approximate and actual
	novelty disagree.
	
	\subsubsection{Impact of Sampling.}
	\label{section:impact_sampling}
	We proceed now to derive the probability of error induced by sampling from $\beta_l(s)$ following the
	discrete uniform distribution. By \textit{error} we refer to the event of $w(s)$ $\neq$ $\hat{w}(s)$, that is 
	$\hat{w}(s)$ is greater or less than $w(s)$ for a state $s$. We start by defining the probability $\bm{\gamma_t}$, given $s$ and ${\cal P}(s, H)$, of a particular tuple $t\in\beta_l(s)$ observed as \textit{new}  in $s$, that is $t \notin \bigcup_{s'\in {\cal P}(s,H)}Z_l(s')$ and $t\in Z_l(s)$, as
	\begin{align}
	\gamma_t =  \left(\prod_{s' \in {\cal P}_t '(s)}\left( 1 - \frac{z}{|\beta_l(s')|} \right)\right)\,\frac{z}{|\beta_l(s)|} \label{p_gamma}
	\end{align}
	\noindent
	where ${\cal P}_t'(s)$ is defined as the set $\{ s'\,|\, s' \in {\cal P}(s, H), s' \models t\}$, $z$ is the sample size, $z/|\beta_l(s)|$
	follows from the probability mass function $p_X$. The event of taking a sample at $s$ is \textit{independent} from that of sample at $s'\in {\cal P}(s, H)$ which allows us to use the \textit{product rule}. 
	Also, 
	it follows that as $z\to |\beta_l|$, the probability $\gamma_t\to 0$ when ${\cal P}_t'(s)$$\neq$$\emptyset$, and  $\gamma_t\to 1$ when ${\cal P}_t'(s)$$=$$\emptyset$. That is, if all tuples $t\in \beta_l(s')$ are sampled in each $s' \in S$ , as we do when computing novelty exactly, the probability of $t$ being \textit{new}  in $s$ is $0$, if $t \in \bigcup_{s'\in {\cal P}(s,H)}Z_l(s')$, and $1$ otherwise.

	From Equation~\ref{p_gamma}, we follow that the probability of tuple $t$ \textit{not} being \textit{new} is $\left[1 -  \gamma_t \right]$. We use this result to compute the probability that none of the tuples $t\in\beta_l(s)$ are \textit{new}, assuming \textit{independence} between different tuples {to make the derivation tractable}, as
	
	\begin{align}
	p_{l} = \prod_{t \in  \beta_l(s)}\left[1 -  \gamma_t \right] \label{p_l}
	\end{align}
	
	\noindent Using Equation~\ref{p_l}, we can now define the probability of approximate novelty measure to be greater or smaller than actual, $P_H = P(\hat{w}(s)>w(s))$, $P_L = P(\hat{w}(s)<w(s))$ respectively, as
	
	\begin{align}
	P_{H} = \prod_{i=1}^{w(s) } p_i ,~~P_L = \left(1- \prod_{i=1}^{w(s)-1} p_i\right)  \label{P_LH}
	\end{align}
	
		\noindent Finally, the probability of approximate and actual novelty measures to agree, $P_C = P(\hat{w}(s)=w(s))$, is :
	
	\begin{align}
	P_{C} = \left( \prod_{i=1}^{w(s)-1} p_i  \right)\;\left(1-p_{w(s)}\right) \label{P_C}
	\end{align}
	
	In Lemma~\ref{lem:accum_prob} we prove the sum of probabilities in Eqs.~\ref{P_LH} and \ref{P_C} to be $1$,
	and where $(P_L+P_H)$ is the total probability of error induced by sampling. 
	
	\begin{lemma} 
		The sum of probabilities $P_L$, $P_H$ and $P_C$ is $1$.
		\label{lem:accum_prob}
	\end{lemma}
	\begin{proof}
		Let $q = \prod_{i=1}^{w(s)-1}p_i$ be the probability of not finding new tuples with novelty below $w(s)$. Then, Equations are rewritten as $P_L = 1 - q$, $P_H = q\,p_{w(s)}$ and $P_C=q\,(1 -p_{w(s)})$ for a given novelty $w(s)$. Hence, $P_L + P_H + P_C = (1-q) + q \, p_{w(s)} + q\,(1-p_{w(s)}) = 1 - q + q \, p_{w(s)} + q - q \, p_{w(s)}  = 1 $, proving the correctness of equations.
	\end{proof}

	\subsubsection{Synergies between Sampling and Bloom Filters.}
	\label{section:bloom_filter_config}
	The number of generated states ${\cal P}(s)$ before state $s$ can be $O(b^d)$, that is, exponential on the branching factor $b$ of the transition system $S(P)$ and the length $d$ of the path to $s$ from $s_0$. 
	Therefore, replacing ${\cal N}(s)$ = $\bigcup_{s'\in {\cal P}(s,H)}$ $\beta_l(s')$ by a Bloom filter with $r$ entries cannot come for free. 
	While the Bloom filter will always give the correct answer to membership queries for tuples $t$ that are in ${\cal N}(s)$, it can produce false positives for membership, as it incorrectly gives a positive answer for tuples $t'$ $\notin$ ${\cal N}(s)$. 
	
	We note that sampling from $\beta_l$ enables the use of Bloom filters to ``approximate'' ${\cal N}(s)$. This
	is because it leads to a reduction of the probability of false positives $P_f$ given in 
	Equation~\ref{eq:bloom_filter}, in comparison with what we would obtain from using $\beta_l$ directly as
	in the algorithm for $w(s)$ given in Section~\ref{section:bfws}. This observation follows from noting that
	$q$ in Equation~\ref{eq:bloom_filter} is the expected number of distinct tuples $t$ sampled during the search,
	and the rate of growth of this random variable is directly proportional to $\bar{Z}$. The smaller $\bar{Z}$ is,
	the slower $q$ will grow. 
	From Equation~\ref{eq:bloom_filter},  we can see that the probability
	of false positives $P_f$ increases with the ratio $q/r$, which in turn depends only on $q$
	as $r$ is a constant. Therefore, the rate of growth of $P_f$ depends on $\bar{Z}$. 
	
	Finally, we note that $P_f$ will be maximized when $q$ is exponential on $l$, the maximum size of the tuples
	considered. Then, in principle, false positive probability increases too as $l$ grows larger.
	
\subsubsection{Total probability of error.} 
	
	To obtain the total probability of erroneously appraising the novelty of a newly generated state
	$s$, $P_{error}$, we incorporate Equation~\ref{eq:bloom_filter} into Equation~\ref{p_gamma}
	\begin{align}
	\gamma_t = \left[ \left(\prod_{s' \in {\cal P'}(s)}\left( 1 - \frac{z}{|\beta_l(s')|} \right)\right)\,
	\frac{z}{|\beta_l(s)|}\,
   \left(1-P_f\right)\right] \label{p_i_total}
	\end{align}
	\noindent used to evaluate Equation~\ref{P_LH}, from which then it
	follows that $ P_{error} = P_L + P_H$.
	
	\subsubsection{{Conjoining BFWS($f_5$) and Novelty Approximation.  } }
 	
 	Novelty approximation using sampling and Bloom Filter can be directly applied to BFWS. We replace the $w(s)$ in BFWS($f_5$) with the approximation $\hat{w}(s)$ resulting in $\hat{f}_5 = \langle \hat{w}, \#g\rangle$.

  Additionally,  any reasonable implementation to compute ${\hat w}(s)$ $=$ ${\hat w}_{\langle\#g,\#r\rangle}(s)$ for BFWS($\hat f_5$), as per Definition~\ref{def:approx_novelty},  needs to track the evaluations of partition functions $H$ $=$ $\{\#g,\#r\}$ for all observed tuples. This increases the space complexity by a factor of number of possible partitions, $|G| \cdot |F|$.
   We manage the increase in space complexity by employing a set of Bloom filters, a \textit{bank} $\bm{V}$, and then bounding
   the space available for novelty computation by a parameter $\bm{D_{max}}$. In case $D_{max}$ is sufficiently large to track all the tuples and evaluations of partition functions, we enable exact item membership tests. Otherwise, we use the \textit{bank} of Bloom filters.
   Whenever a new partitioned space is observed, we assign it a Bloom filter from $V$. If the number of observed partitions exceeds $|V|$, we overlap them randomly, allowing different partitions to use the same Bloom filter.  This results in a gradual decrease in the accuracy of novelty computation in exchange for space, the $P_{error}$ increases as more partitions overlap. 
   
   The resulting planner BFWS(${\hat f}_5$)  has the following hyperparameters, namely, the sample size $\bar{Z}$, the size of a Bloom filter $r$ and the bound of space $D_{max}$.  In Section~\ref{experiments}, we present experimental evaluations with different choices of parameter values. 
   
   	\subsubsection{{Increasing the novelty bound.} }
   As discussed in  Section~\ref{section:bfws}, any implementation of Definition~\ref{def:novelty} has a complexity of $O(|F|^{i})$ rendering the computation impractical for many instances. Whereas, Definition~\ref{def:approx_novelty} has linear complexity allowing us to compute ${\hat w}(s)$ for any value of $i \in [1, |F|]$. In the following section, we describe the impact of increasing ${\cal W} \in [1,  i+1]$ in the \emph{polynomial} planners: BFWS($f_5$) with novelty pruning \cite{lipovetzky:17:icaps}.
	\begin{theorem} 
	Let $P$ $=$ $\langle F, O, I, G\rangle$ be a STRIPS planning problem. The number of nodes \textit{generated} for each novelty category $\omega \in {\cal W}$, when run with $P$ as input, is less than or equal to  ${\binom{|F|}{\omega}} \cdot |G| \cdot |F|$.
	\label{theorem:num_novelty_nodes}
	\end{theorem}
	\begin{proof}
		The upper bound on the number of observed state partitions is given by $|G| \cdot |F|$. Also, the count of tuples of size $\omega$, of atoms in $F$, is ${\binom{|F|}{\omega}}$. Hence,  the number of nodes with $w(s)$ $=$ $\omega$ cannot exceed ${ \binom{|F|}{\omega}} \cdot |G| \cdot |F|$.
	\end{proof}
 
 	\begin{corollary} 
 		Let $P$ $=$ $\langle F, O, I, G\rangle$ be a STRIPS planning problem, and BFWS(${\hat f}_5$) the polynomial planner using Bloom filters introduced above. When run with $P$ as its input, BFWS(${\hat f}_5$) generates at most $|V| \cdot {r}$ nodes for each novelty category $\omega \in {\cal W}$.
 		\label{cor:num_exp_gen_bloom}
	 \end{corollary}
	 
	 \begin{proof}
	 A Bloom filter represents ${\binom{|F|}{\omega}}$ tuples, but the number of true negatives is bound by the size of Bloom filter $r$. Also, we create a set of Bloom filters, the \textit{bank} $V$, that represents the set of partitioned spaces of cardinality $|G| \cdot |F|$. Hence, the above bound holds.
	 \end{proof}
	
	From Theorem~\ref{theorem:num_novelty_nodes}, we note that the bound on number of nodes with $w(s)$ $=$ $\omega + 1 $  increases by $O(|F|)$ in comparison to those with $w(s)$ $=$ $\omega$ , which makes nodes with large value of ${\hat w}(s)$ unlikely candidates for expansion.
		This leads us to another important issue afflicting BFWS algorithms, 
	inherited from BFS, i.e.	only a small fraction of the nodes that 
	make it into the open list are ever considered for expansion. 
	 %
	 One possible method to address this challenge is to choose small size for $r$ and $V$, and from 
	 Corollary~\ref{cor:num_exp_gen_bloom} we can deduce that it will bound the nodes in each novelty category by $|V| \cdot r$, hence there are less nodes in each category.  This makes it seem that nodes with high novelty $w(s)$ are now more likely to expand, however, in practice  nodes receive a higher approximate novelty value  ${\hat w}(s)$  because of increase in $P_{f}$.  In the following section, we discuss a method that remediates this.
	  
	\section{Best First Search with Open List Control}
	\label{olc}
	
	We propose an example of a  novel methodology to design 
	BFS algorithms that aim at controlling the rate of growth of {nodes of each category $\omega \in {\cal W}$ in} the open list. 
	To do this, we model the search as a discrete-time dynamical
	system subject to perturbation, and an optimal control problem~\cite{bertsekas:17} is 
	formulated where optimal policies ensure that a rate of growth less than the
	branching factor $b$ of $S(P)$ is sustained. With some simplifying assumptions,
	the optimal policy for this control problem can be derived analytically, as shown below, and
	integrated directly into the search algorithm.
	We model the evolution over time of the internal state (i.e. size of open lists) of a BFWS-like algorithm $B$, subject to function  $T_B$, abstracting the instructions
	executed in one iteration of
	the expansion loop of $B$, as the dynamical system
	\[ x_{k+1} = T_B(x_k, {u}_k, c_k) \]
	where, $\bm{k}$ is the index of current expansion, $\bm{x_{k}}$ 
	 is a suitably defined abstraction 
	of the internal state of the search algorithm $B$ at time $k$, 
	$\bm{u_k} \in [0,1)^{|{\cal W}|}$ is the \emph{control action}, 
	that prescribes the \emph{pruning rate} for states with $w(s)$ $=$ $\omega$, and
	$\bm{c_k} \in \mathbb{N}^{|{\cal W}|}$ is the count of successor states $s'$ at time $k$ with $w(s')=\omega$. 
	$c_k$ is the \emph{perturbation}, that is, an uncontrollable side-effect of node expansion
	that has been modeled as a uniformly distributed discrete random variable. 
	
	The information we track in $\bm{x_k}$ is given by the tuple $\langle n_e$, $n_v(\omega) \rangle$, where $n_e$ is the number of expanded nodes so far, and $n_v(\omega)$ is the
	count of \emph{novel states} visited for each novelty category $\omega \in {\cal W}$.
	If $u = 0^{|{\cal W}|}$, then $T_B$ is deterministic and $B$ behaves like a standard BFWS
	algorithm. Otherwise, the successors $s'$ of state $s$ pointed at by node $n$ in the
	open list with $\min f(n)$ are generated with
	probability $1 - u^{\omega}$ when $w(s')=\omega$. If some $s'$ is pruned, $n$ is kept in a \emph{holding queue}
	and re-expanded whenever the open list becomes empty.
	An important implementation detail for $B$ is that it needs to maintain 
	$|{\cal W}|$ open lists in parallel, as keeping smaller open lists is often more 
	performant~\cite{burns:12:socs}, also greatly facilitating implementation and 
	computation of states $x_k$.	
	
	In order to formulate an optimal control problem, we need first to specify a cost function. Considering a very large number of stages or expansions, we can reasonably assume that the horizon is infinite, and define the average cost per stage function~\cite{bertsekas:17}, with a policy $\pi=\{\mu_0, \mu_1,...\}$, of the form
	\begin{align} 
	J_{\pi} = \lim_{N\to \infty} \frac{1}{N} \E_{c_k} \left[ \sum_{k=0}^{N-1} g(x_k, \mu_k(x_k), c_k) \right]
	\label{eq:olc_cost}
	\end{align}
	
	\noindent The choice of the cost per stage $g(x_k$, $\mu_k(x_k)$, $c_k)$ 
	is dictated by the need to seek a trade-off between the {number of states in the} open list {with ${\hat w}(s)$ $=$ ${\omega}$} growing too large and missing out 
	useful novel states. We define $g_k$ as follows
	\begin{align}
	g_k =  \sum_{\omega \in {\cal W}} \left[ {c_{k}^{\omega}\;(1-\mu_k^{\omega}(x_k))} + \frac{1}{(1-\mu_k^{\omega}(x_k))}\right]
	\label{eq:olc_reward}
	\end{align}
	\noindent 
	where the expected count of successor nodes, with novelty $\hat{w}(s)=\omega$, added to the open list at time $k$ is given by 
	$c_{k}^{\omega}\,(1-\mu_k^{\omega}(x_k))$,  and the second term is the inverse rate of node generation, as we want every possible value of novelty to be represented in the open list with positive probability.
	Using Equation~\ref{eq:olc_reward} we can rewrite $J_\pi$ as
	\begin{align} 
	J_{\pi} = \lim_{N\to \infty} \frac{1}{N} \E_{c_k} \left[ \sum_{k=0}^{N-1} g_k\right] \label{eq:olc_cost_1}
	\end{align}
	\noindent We make an assumption to facilitate obtaining the optimal policy, namely, $\mu_k$ will converge to some
	stationary $\mu$ as $k \to \infty$, so it can be used to estimate the cost of future stages accurately.
	Also, the expected value of uniformly distributed random variable $c_k^{\omega}$ is calculated
	from $n_v(\omega)$ and $n_e$, as $E\left[c_k^{\omega}\right] = n_v(\omega)/n_e$.
	It follows then from Equation~\ref{eq:olc_cost_1} that
	\begin{align}
	J_{\pi} =  \sum_{\omega \in {\cal W}} \left(\frac{n_v(\omega)\,(1-\mu^\omega(x_k))}{n_e} + \frac{1}{\left(1-{\mu^\omega(x_k)}\right)}\right)
	\label{eq:olc_final_cost}
	\end{align}
	
	\noindent We note that $J_{\pi}$ is strictly convex and differentiable over 
	$\mu^\omega \in [0,1)$, and optimal values for the control inputs correspond with optimal
	solutions of the optimization problem
	\[ \min_{\mu \in [0, 1)^{|\cal W|}} J_\pi \]
	
	\noindent Such a solution is directly obtained from Equation~\ref{eq:olc_final_cost} from
	the solution of the differential equation $\partial J_\pi / \partial \mu^{\omega} = 0$,
	\begin{align} \label{opt-policy}
	\mu^\omega(x_k) =
	\begin{cases} 
	 {1 - \left(\frac{n_e}{n_v(\omega)}\right)^{\frac{1}{2}}}, &  \text{if } n_e/n_v(\omega)<1\\
	 0, 																														& \text{otherwise}
	\end{cases}
	\end{align}

	{ Note that the \textit{holding queue} follows the well-established practice of segmenting the search frontier into multiple queues \cite{richter2010lama}. We use the optimal policy that we derived above to control the different queues based on their $w(s)$ values, ensuring that each queue represents every category $\omega \in {\cal W}$ with positive probability. The queues are accessed sequentially, expanding all the nodes in the current queue before switching to the next. Also, the way we implement it, the queues are generated lazily, following a partial expansion of nodes whose successor falls in the subsequent queues.}
	
	\section{Experimental Evaluation}
    \label{experiments}
	
	In order to evaluate the impact \textit{novelty approximation} and \textit{open list control} has on width-based planners, we implemented different instantiations of BFWS($f_5$): \textit{complete} as described in the Section~\ref{section:background}, or \textit{incomplete} if nodes with novelty greater than a given bound are pruned \cite{lipovetzky:17:icaps}. We used the \textit{Downward Lab}'s experiment module \cite{seipp:lab} on a server with Intel Xeon Processors (2 GHz) with a $1800\;sec$ and $8\;GB$ \textit{time} and \textit{memory} limit, respectively. All BFWS planners are implemented in C++ using the planning modules from \textit{LAPKT} \cite{lapkt} and grounder from \textit{Tarski} \cite{tarski}. We use every benchmark in the IPC \emph{satisficing} track to evaluate the correctness of the novelty approximation $\hat w$, and performance of new planners that use $\hat w$. In case a domain has appeared over multiple IPCs, we used the problem set from the most recent IPC. We compare our new planners against notable \emph{polynomial} planners: BFWS($f_5$) with novelty pruning and $\langle1,2$-C-M$\rangle$, a sequential polynomial planner  \cite{lipovetzky:17:icaps}, as well as two state-of-the-art planners {DUAL-BFWS} \cite{lipovetzky:17:aaai} and LAMA-\textit{first} \cite{richter2010lama}. We show that the introduction of these methods has a significant impact on the performance of the BFWS algorithms.

\subsubsection{Correctness of novelty approximation.}

 We evaluate the reliability of the novelty approximation by observing the effect on rate of correct and incorrect (\emph{lower} or \emph{higher}) approximation of novelty over varying sizes of sample $\bar Z$  and Bloom Filter $r$,  scaled by a multiplicative factor $\delta$. The novelty approximation $\hat w$ is \emph{correct} or \emph{accurate} if $\hat w(s)$$=$$w(s)$. We limit the maximum size of tuple evaluated to 3, as higher order computations for exact novelty $w$ were infeasible within the practical constraints of time and memory. Thus, $w: S \rightarrow {\cal W}$, where  $\cal{W}$$=$$[1 , 4]$, and $w$$=$$4$ represents all nodes with $w$$>$$3$ .  To distinguish the impact of sampling from that of Bloom Filter, we capture the results of novelty approximation \emph{with} and \emph{without} Bloom Filter, hereafter, represented as $\hat w$ and $\bm{\hat w_{\bar b}}$, respectively. We capture the statistics from 1200 solved instances in IPC satisficing benchmarks.

\begin{figure}[t]
	\includegraphics[width=\linewidth]{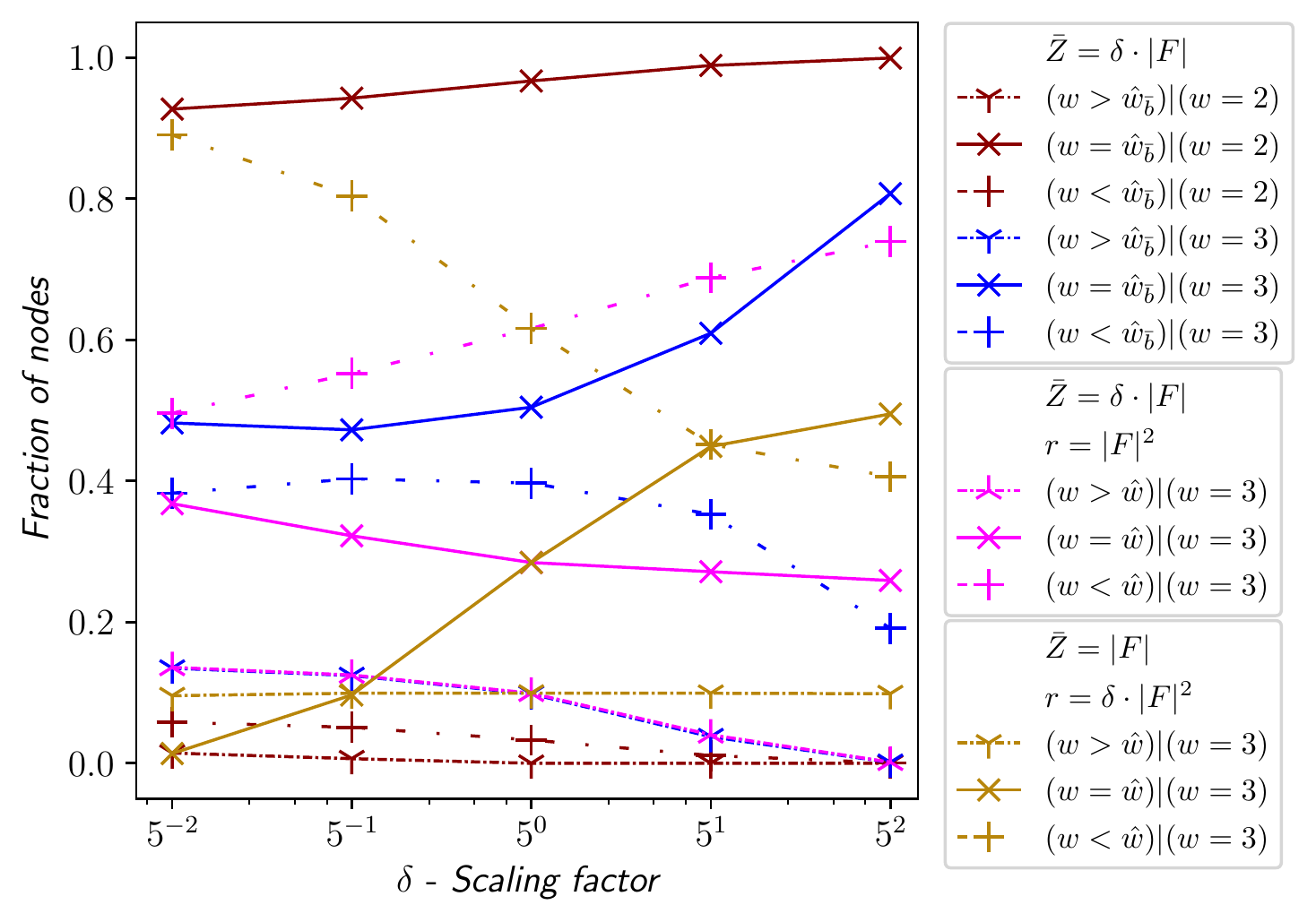}
	\caption{Variation in the rate of \emph{accurate}  ($w=\hat w$), \emph{lower} ($w> \hat w $)  and \emph{higher} ($ w < \hat w$) approximation of novelty $w$ over different sizes of sample($\bar Z$) and Bloom Filter($r$).}
	\label{fig1}
\end{figure}

From Fig.~\ref{fig1}, we note that the rate of \emph{correct} approximate novelty ($w = \hat w_{\bar b}$) increases with sample size $\bar Z$, when Bloom Filters are not used. This backs up our analysis in Section~\ref{p-error} that the accuracy of novelty approximation is likely to increase with sample size. We also observe that the rate does not decrease below $1/2$ even for $w=3$, where the sample is order of $1/|F|^2$ smaller than the exhaustive set. This is a significant improvement over a trivial method of using a coin toss to determine whether or not a tuple $t$ of size $l$ is \emph{new} in state $s$, which has probability $1/8$ of selecting \emph{correct} novelty, given $w$$=$$3$ .

While we note that the novelty approximation without Bloom Filters performs satisfactorily in terms of correctness, it is still  infeasible to store exhaustive set of tuples $\beta_l(s)$ of size $O(|F|^l)$, when $l$$\ge$$3$, for many IPC problems. As discussed in Section~\ref{p-error}, we address this by using Bloom Filters for evaluation of $\hat w \ge3$. With this addition, we observe a slight decrease in the rate of \emph{correct} novelty approximation, which is the consequence of false positives, discussed in Section~\ref{section:bloom_filters}. Also, we observe that the trend along sample size is reversed, i.e.,  the rate of \emph{correct} novelty approximation now decreases with increase in sample size. The trend is in line with the theoretical analysis in the Section \emph{Synergies between Sampling and Bloom Filters}. On the other hand, increasing the size of Bloom Filter $r$ improves the results, as the false positive rate decreases.

\subsubsection{Performance over benchmarks.}

\begin{figure}[t]
	\centering
	\includegraphics[width=0.85\linewidth]{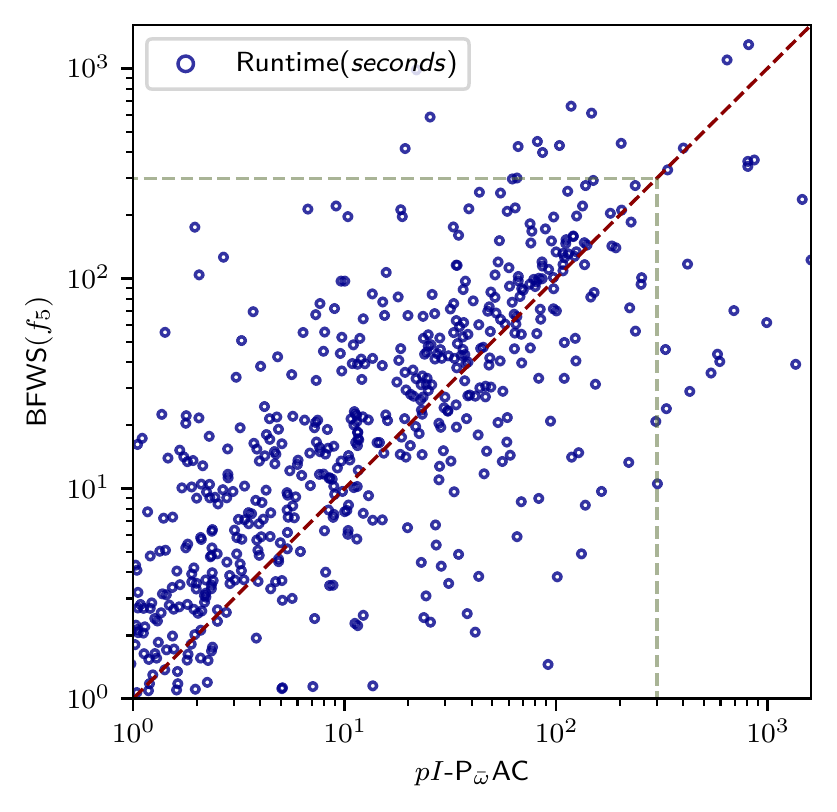}
	\caption{Pairwise comparison of runtime, over All IPC satisficing benchmarks, between BFWS($f_5$) and $pI$-P$_{\bar\omega}$AC.}
	\label{fig6}
\end{figure}

Hereafter, we represent a particular configuration of BFWS planner as '$\bm{pI\text{-(P}\vert {\cal L}}$)$\bm{_{\bar\omega}\text{AC}}$'. The \textit{prefix} '$p$-' refers to the use of novelty based pruning for nodes with $\hat w(s)$ $>$ $\bar{\omega}.$ '$I$-' refers to BFWS called sequentially until the problem is solved over $\bar\omega \in$ $[1,|F|]$. P$_{\bar \omega}$ refers to \textit{BFWS($f_5$)} planner with the set of possible novelty categories $\cal W$ $=$ $[1,\bar\omega+1]$. ${\cal L}_{\bar \omega}$ refers to BFWS($f_5$) with the goal counting heuristics replaced by landmark counts \cite{RichterHelmertWestphal:08:aaai}, that is $f_{\cal L}$=$\langle w,h_L \rangle$. 'A' denotes that $\hat{w}(s)$ is used instead of $w(s)$, and 'C' denotes that BFWS is modified to control open list growth as described in Section~\ref{olc}. All 'AC' planners were run 4 times with different seeds, so we report the \textit{mean and standard deviation} of statistics of interest.
	 
	 	{We set the sample size $\bar{Z}$$=$$ |F|$ so as to maintain a linear time complexity. We found that $D_{max}$ values between 100 MB and 1GB had similarly good results for '$pI$-P$_{\bar \omega}$AC', we show the results for $D_{max}$ $=$ $500~MB$. For the Bloom Filters size $r$, we didn't observe much variation between $100~KB ~/~8\cdot10^5$ $bits$ and $10~MB~/~8\cdot10^7$ $bits$, with $D_{max}$ $=$ $1~GB$. In our final implementation, we set an initial value of $r$ $=$  $|F|^2$, subject to increase when ${\bar \omega}r \cdot |V|$ $<$  $D_{max}$ $\land$ $|V|$ $=$ $|G| \cdot |F|$, and decrease when ${\bar \omega}r$  $>$  $D_{max}$ $\land$ $|V|$ $=$ $1$. A total of 103 instances out of 1691 used the  \textit{bank} of Bloom filters $V$, described in Section~\ref{def:approx_novelty}, ensuring that novelty computation does not exceed $D_{max}$. Lastly, we use the solution to the problem of minimizing $P_f$ in Equation~\ref{eq:bloom_filter} to chose the number of hash functions as $K$$=$$\ln2~(r/q)$,  where $q$ $=$ $\binom{|F|}{\omega}$.}
	
	\begin{table}[t]
		\centering
		{
			\begin{tabular}{
					@{}l|%
					@{\extracolsep{2pt}}r|%
					@{\extracolsep{2pt}}r|%
					@{\extracolsep{2pt}}r|%
					@{\extracolsep{2pt}}r%
					@{\extracolsep{2pt}}r%
					@{\extracolsep{2pt}}r%
					@{\extracolsep{2pt}}r%
					@{}
				}
				\toprule
				{} &      $\bar w=1$ &      $\bar w=2$ &      $\bar w=3$ &     $\bar w\ge 4$ \\
				\midrule
				\emph{\# Instances} &  100.00 \% &  18.92\%  &  3.68\%  &  1.05\%  \\
				\bottomrule
				
			\end{tabular}
		}
		\caption{\% of instances across all IPC satisficing benchmarks where a node of novelty $\bar w$ was recorded in found plans.} \label{table2}
	\end{table}

	Looking back at the motivation, a key driver for introducing the novelty approximation was to enable novelty computation for values greater than 2, which was infeasible for many IPC domains with the exact novelty definition.   The results for $p$-P3A in Table~\ref{table3}, show that our hypothesis was indeed correct as computing higher novelties with approximation improves coverage. This is substantiated in Table~\ref{table2} which shows that $\approx5\%$ of the solved instances had one or more nodes with $\bar w(s) \ge 3$ in the solution plan. Moreover,  the coverage of approximate planners with $\bar \omega$ $=2$, P$_2$A and $p$-P$_2$A, improves in comparison to P$_2$ and $p$-P$_2$, respectively, which indicates that there is no apparent demerit of using novelty approximation. The improvement can be attributed to \emph{polynomial} time and space complexity of $\hat w(s)$ allowing for additional search capacity.
	
{Though  BFWS($\hat{f}_5$) performs satisfactorily, it has a key shortcoming which impacts the search within the limited time environment, i.e. for large instances of domains with width $i$ $>$ $2$, the BFWS search driven by the evaluation function $f_5$ $=$ $\langle w,\#g\rangle$ exhausts all the available time in expanding nodes with $w(s)$ $\leq$ $2$. Moreover, the issue gets compounded for domains with high branching factor as the open list doesn't fit within the memory bounds. We address both the issues by applying the open list control discussed in Section~\ref{olc}. In our implementation, the \textit{control} is not applied to child nodes with novelty $w(s)$$=$$1$, as the maximum count of such nodes is small, $O(\vert F \vert)$,  and have minimal impact on space.  Note that this method will not cause the search to become incomplete. However, if we choose not to maintain the \textit{holding queue}, we get a search that is \textit{incomplete} and terminates early. 	Introducing the \textit{open list control} in BFWS ($\hat{f_5}$) leads to noticeable improvement in coverage of  P$_2$AC and  P$_3$AC which can be observed in Table~\ref{table3}. We do not report tables on plan length due to space limits as plan length remains similar for all configurations.}
		
	\begin{figure}[t]
	\includegraphics[width=0.9\linewidth]{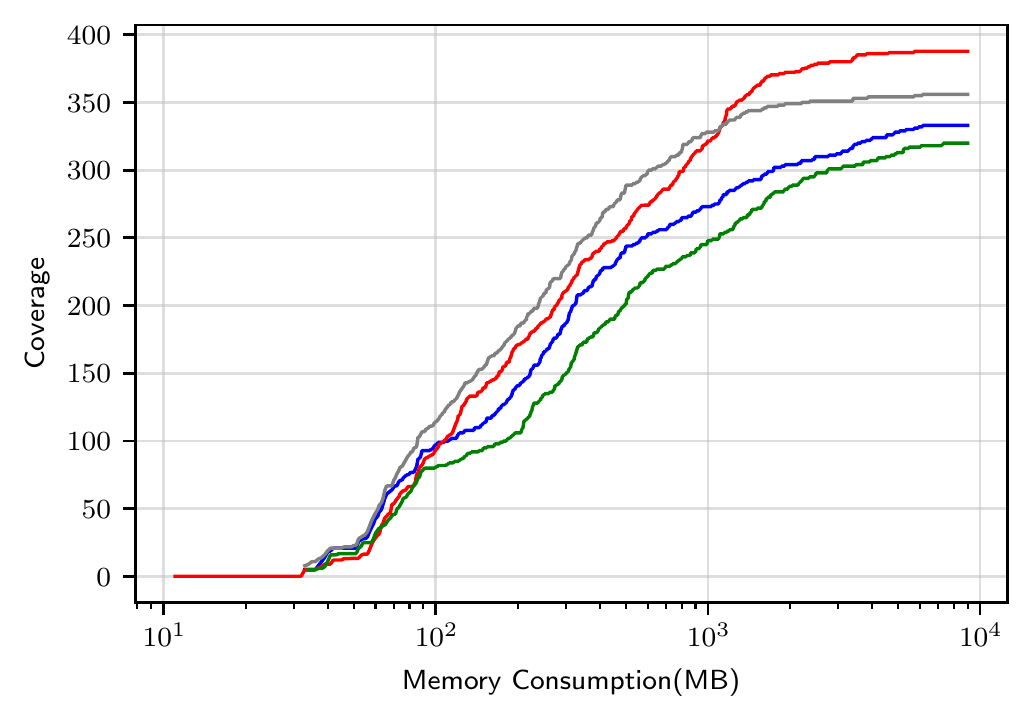}		\includegraphics[width=0.9\linewidth]{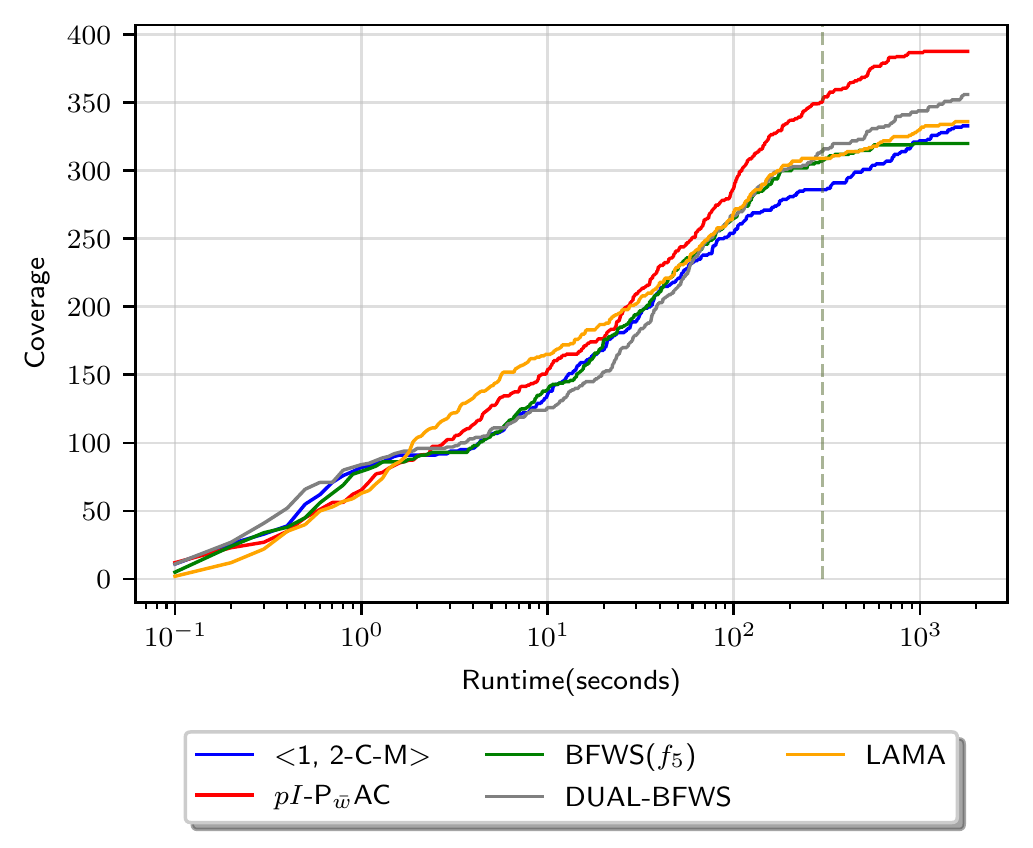}

	\caption{\textit{Coverage} over \textit{memory(MB)} and \textit{time(seconds)} on IPC 2014 and 2018 satisficing benchmarks.}
	\label{fig3}
\end{figure}

\begin{table*}[t] 
	\centering
    \def\arraystretch{0.69}
	{
	\footnotesize
		\begin{tabular}
			{%
				@{}l%
				@{\extracolsep{2pt}}r%
				@{\extracolsep{2pt}}r%
				@{\extracolsep{2pt}}r|%
				@{\extracolsep{2pt}}r%
				@{\extracolsep{2pt}}r%
				@{\extracolsep{2pt}}r%
				@{\extracolsep{2pt}}r||%
				@{\extracolsep{2pt}}r%
				@{\extracolsep{2pt}}r|%
				@{\extracolsep{2pt}}r%
				@{\extracolsep{2pt}}r%
				@{\extracolsep{2.5pt}}r%
				@{\extracolsep{2.5pt}}r%
				@{\extracolsep{2pt}}r%
				@{\extracolsep{2.1pt}}r%
				@{\extracolsep{2.1pt}}r%
				@{}}
			\toprule
			{domain} &  B1 & B2 & P$_2$ &           P$_2$A &          P$_2$AC &   P$_3$A &       P$_3$AC & $p$-P$_2$ & B3 &       $p$-P$_2$A &       $p$-P$_3$A &      $pI$-P$_{\bar\omega}$AC & $pI$-${\cal L}_{\bar\omega}$AC  \\
			\midrule
			agricola (20)                &    12 &         8 &      11 &    11$\pm$1.0 &    12$\pm$1.7 &    15$\pm$0.8 &     {\textbf{16$\pm$1.3 }} &     10 &        15 &    10$\pm$0.6 &    15$\pm$0.8  &           {\textbf{16$\pm$1.3}}  &       12$\pm$0.5 \\
			airport (50)                      &    34 &       {\textbf{47}}   &      46 &    46$\pm$0.6 &    46$\pm$0.6 &    44$\pm$0.0 &    44$\pm$0.6 &     46 &         {\textbf{47}}   &    46$\pm$0.6 &    45$\pm$1.0 &         46$\pm$0.6 &       46$\pm$0.5 \\
			assembly (30)                     &   {\textbf{30}}  &        {\textbf{30}}  &      {\textbf{30}}  &    {\textbf{30$\pm$0.6}}  &    29$\pm$1.0 &    {\textbf{30$\pm$0.6}}  &    {\textbf{30$\pm$0.6 }} &     {\textbf{30}}  &       {\textbf{30}}  &    {\textbf{30$\pm$0.5}}  &    {\textbf{30$\pm$0.0}}  &           {\textbf{30$\pm$0.0 }} &       {\textbf{30$\pm$0.0}}  \\
			caldera (20)                &    16 &        {\textbf{20}}  &      15 &    16$\pm$1.0 &    {\textbf{20$\pm$0.5}}  &    16$\pm$1.0 &    18$\pm$0.5 &     19 &       {\textbf{20}}   &    {\textbf{20$\pm$0.5}}  &    18$\pm$0.5 &            {\textbf{ 20$\pm$0.5 }}&       {\textbf{20$\pm$0.0}}  \\
			cavediving (20)                   &     7 &         7 &       7 &     7$\pm$0.0 &     8$\pm$0.6 &     8$\pm$0.0 &     8$\pm$0.5 &      1 &         8 &     2$\pm$2.9 &     8$\pm$0.5 &             {\textbf{9$\pm$1.0}}   &        8$\pm$0.5 \\
			childsnack (20)     &     6 &        {\textbf{10}}  &       0 &     4$\pm$1.3 &     5$\pm$1.7 &     3$\pm$1.9 &     6$\pm$0.6 &      0 &         2 &     5$\pm$0.5 &     6$\pm$0.6 &              8$\pm$1.3 &        8$\pm$1.3 \\
			citycar (20)          &     5 &        {\textbf{20}}  &       5 &     5$\pm$0.0 &    {\textbf{20$\pm$0.0}}  &     5$\pm$0.0 &    {\textbf{20$\pm$0.6}}  &     {\textbf{20}}  &       {\textbf{20}}   &   {\textbf{20$\pm$0.5}}  &     5$\pm$0.0 &            {\textbf{20$\pm$0.0}}  &       {\textbf{20$\pm$0.0 }} \\
			data-network (20)          &    13 &        11 &       9 &    12$\pm$1.7 &    {\textbf{19$\pm$1.0}} &    11$\pm$2.5 &    18$\pm$0.5 &     16 &        14 &    17$\pm$1.0 &    16$\pm$0.6 &            18$\pm$0.6 &       18$\pm$0.6 \\
			depot (22)                        &    20 &        {\textbf{22}}  &     {\textbf{22}}   &   {\textbf{22$\pm$0.0}}  &    {\textbf{22$\pm$0.0}} &    {\textbf{22$\pm$0.0}}  &    {\textbf{22$\pm$0.0}} &     {\textbf{22}}  &        {\textbf{22}}  &    {\textbf{22$\pm$0.0}}  &       {\textbf{22$\pm$0.0}}  &        {\textbf{22$\pm$0.0}}  &      {\textbf{22$\pm$0.0}}  \\
			flashfill (20)            &    14 &        {\textbf{16}}  &      12 &    14$\pm$2.4 &    14$\pm$0.6 &    14$\pm$2.2 &    14$\pm$1.0 &     15 &         9 &    14$\pm$2.4 &    14$\pm$1.9 &            14$\pm$1.0 &       14$\pm$1.0 \\
			floortile (20)       &     {\textbf{2}}  &         {\textbf{2}}  &       1 &     {\textbf{2$\pm$0.}} 5 &     {\textbf{2$\pm$0.0}} &     {\textbf{2$\pm$0.0}} &     {\textbf{2$\pm$0.0}}  &      0 &         1 &     1$\pm$0.0 &     {\textbf{2$\pm$0.5}}  &              {\textbf{2$\pm$0.0}}  &        {\textbf{2$\pm$0.0}}  \\
			hiking (20)          &    {\textbf{20}}  &        12 &      12 &    14$\pm$2.1 &     8$\pm$0.8 &    18$\pm$1.0 &    {\textbf{20$\pm$0.5}}  &      9 &        13 &    12$\pm$1.8 &    {\textbf{20$\pm$0.0}} &          {\textbf{20$\pm$0.0}}  &       19$\pm$0.5 \\
			maintenance (20)       &    11 &        {\textbf{17}}  &      {\textbf{17}}  &    16$\pm$0.5 &    16$\pm$0.6 &    16$\pm$0.5 &    {\textbf{17$\pm$0.5}}  &     {\textbf{17}}  &        {\textbf{17}}  &    16$\pm$0.5 &    16$\pm$0.5 &          {\textbf{17$\pm$0.5}}   &      {\textbf{17$\pm$0.5}}  \\
			mprime (35)                      &    {\textbf{35}}  &        {\textbf{35}}  &      32 &    30$\pm$0.6 &    {\textbf{35$\pm$0.0}}  &    31$\pm$0.5 &    34$\pm$0.8 &     {\textbf{35}}  &       {\textbf{35}}   &    {\textbf{35$\pm$0.0}}  &       32$\pm$0.8 &        {\textbf{35$\pm$0.0}}  &       {\textbf{35$\pm$0.0}}  \\
			mystery (30)                      &    {\textbf{19}}  &       {\textbf{19}}   &    {\textbf{19}}  &    {\textbf{19$\pm$0.0 }} &    {\textbf{19$\pm$0.0}}  &    {\textbf{19$\pm$0.5}}  &    {\textbf{19$\pm$0.5}}  &     {\textbf{19}}  &        18 &    {\textbf{19$\pm$0.0}} &    {\textbf{19$\pm$0.5}} &        {\textbf{19$\pm$0.5}}  &      {\textbf{19$\pm$0.5}}  \\
			nomystery (20)       &    11 &        {\textbf{19}}  &      13 &    14$\pm$1.0 &    12$\pm$1.0 &    13$\pm$0.5 &    14$\pm$1.0 &     13 &        13 &    12$\pm$1.7 &    14$\pm$0.6 &          15$\pm$1.0 &       18$\pm$1.4 \\
			nurikabe (20)               &     9 &        14 &      {\textbf{16}}  &    14$\pm$0.6 &    15$\pm$1.3 &    14$\pm$0.6 &    15$\pm$2.1 &     {\textbf{16}}  &        {\textbf{16}}  &    14$\pm$0.5 &    14$\pm$0.6 &           15$\pm$1.0 &       14$\pm$0.6 \\
			org-synth-split (20)&    {\textbf{12}}  &        11 &       5 &     6$\pm$0.5 &     3$\pm$0.8 &     7$\pm$1.0 &     5$\pm$1.4 &      4 &         3 &     4$\pm$0.5 &     6$\pm$1.0 &            7$\pm$0.0 &        6$\pm$0.5 \\
			parcprinter (20)     &    {\textbf{20}}  &        16 &       9 &     5$\pm$1.0 &     5$\pm$1.9 &     5$\pm$0.8 &     6$\pm$1.0 &      9 &        16 &     6$\pm$1.0 &     5$\pm$1.0 &             8$\pm$0.0 &        6$\pm$0.5 \\
			pathways-neg (30)               &    24 &        {\textbf{30}}  &      23 &    {\textbf{30$\pm$0.6}}  &    29$\pm$1.5 &    {\textbf{30$\pm$0.6}}  &    29$\pm$0.5 &     24 &        27 &    {\textbf{30$\pm$0.5}} &        {\textbf{30$\pm$0.6}} &        {\textbf{30$\pm$0.0}}  &       {\textbf{30$\pm$0.0}} \\
			pegsol (20)         &     {\textbf{20}}  &         {\textbf{20}}  &       {\textbf{20}}  &    {\textbf{20$\pm$0.0 }} &     {\textbf{20$\pm$0.5}}  &     {\textbf{20$\pm$0.0}} &     {\textbf{20$\pm$0.0}} &      5 &         {\textbf{20}}  &    12$\pm$1.5 &    18$\pm$0.5 &            {\textbf{20$\pm$0.0}}  &        {\textbf{20$\pm$0.0}} \\
			pipesworld-nt (50)         &    43 &        {\textbf{50}}  &       {\textbf{50}}  &     {\textbf{50$\pm$0.0}}  &     {\textbf{50$\pm$0.0}}  &    {\textbf{50$\pm$0.0 }} &     {\textbf{50$\pm$0.0}}  &      {\textbf{50}}  &         {\textbf{50}}  &     {\textbf{50$\pm$0.0}}  &     {\textbf{50$\pm$0.0}} &           {\textbf{50$\pm$0.0}} &        {\textbf{50$\pm$0.0}}  \\
			pipesworld-t (50)           &     {\textbf{43}}  &        38 &       {\textbf{43}}  &    42$\pm$0.5 &    42$\pm$0.6 &    42$\pm$0.5 &     {\textbf{43$\pm$0.8}}  &     41 &        39 &    41$\pm$1.5 &    42$\pm$0.5 &         42$\pm$1.2 &        {\textbf{43$\pm$1.5}}  \\
			psr-small (50)                    &     {\textbf{50}}  &        {\textbf{50}}   &      48 &    49$\pm$0.5 &    49$\pm$0.5 &     {\textbf{50$\pm$0.0}}  &    {\textbf{50$\pm$0.0 }} &     31 &        46 &    34$\pm$1.3 &    43$\pm$0.8 &            49$\pm$0.5 &       48$\pm$0.6 \\
			rovers (40)                       &    {\textbf{40}}   &        37 &      39 &     {\textbf{40$\pm$0.0 }} &     {\textbf{40$\pm$0.0}}  &     {\textbf{40$\pm$0.0}}  &     {\textbf{40$\pm$0.0}}  &     39 &        38 &     {\textbf{40$\pm$0.0}}  &     {\textbf{40$\pm$0.0}} &          {\textbf{40$\pm$0.0}}   &       {\textbf{40$\pm$0.0}}  \\
			satellite (36)                    &     {\textbf{36}}  &        31 &      27 &    30$\pm$0.8 &    32$\pm$0.6 &    30$\pm$0.5 &    30$\pm$0.0 &     27 &        31 &    32$\pm$0.5 &    30$\pm$0.5   &     34$\pm$0.6 &       34$\pm$0.6 \\
			schedule (150)                     &    {\textbf{150}}  &       149 &     149 &   149$\pm$1.0 &   149$\pm$0.8 &   149$\pm$1.0 &    {\textbf{150$\pm$0.6}} &    149 &       149 &   149$\pm$1.0 &   149$\pm$1.0 &      149$\pm$1.0 &      149$\pm$1.0 \\
			settlers (20)                &     {\textbf{18}}  &         8 &       7 &     6$\pm$1.0 &    12$\pm$0.6 &     6$\pm$1.3 &    10$\pm$0.0 &     10 &        11 &     9$\pm$1.0 &     6$\pm$1.0 &          12$\pm$0.6 &       17$\pm$1.3 \\
			snake (20)                   &     5 &        12 &      19 &    16$\pm$0.5 &    15$\pm$0.8 &    17$\pm$0.5 &    17$\pm$0.5 &     18 &         3 &    16$\pm$0.5 &    17$\pm$0.5 &          {\textbf{20$\pm$0.5}}   &         {\textbf{20$\pm$0.5}}  \\
			sokoban (20)       &     {\textbf{19}}  &        17 &      14 &    15$\pm$0.5 &    10$\pm$1.0 &    16$\pm$0.5 &    14$\pm$0.8 &      6 &        13 &     4$\pm$0.6 &    11$\pm$0.5 &            16$\pm$0.6 &       16$\pm$0.6 \\
			spider (20)                 &     {\textbf{16}}  &        14 &      13 &    15$\pm$1.0 &    14$\pm$1.0 &    15$\pm$1.0 &    14$\pm$1.3 &     13 &        11 &    15$\pm$1.0 &    15$\pm$1.0 &         14$\pm$1.0 &       15$\pm$0.5 \\
			storage (30)                      &    20 &        28 &      29 &     {\textbf{30$\pm$0.6}}  &     {\textbf{30$\pm$0.6 }} &     {\textbf{30$\pm$0.6}}  &     {\textbf{30$\pm$0.5}}  &      {\textbf{30}}  &        29 &     {\textbf{30$\pm$0.6}}  &     {\textbf{30$\pm$0.6}} &            {\textbf{30$\pm$0.5}}  &        {\textbf{30$\pm$0.6}}  \\
			termes (20)                  &     {\textbf{16}}  &         9 &       9 &    10$\pm$0.0 &     8$\pm$0.6 &     9$\pm$1.0 &     8$\pm$1.3 &      1 &         6 &     2$\pm$0.5 &     6$\pm$1.9 &          7$\pm$1.4 &       10$\pm$1.3 \\
			tetris (20)           &    16 &        16 &       {\textbf{20}}  &     {\textbf{20$\pm$0.0}} &     {\textbf{20$\pm$0.0}}  &     {\textbf{20$\pm$0.0}} &    {\textbf{20$\pm$0.0}}   &     {\textbf{20}} &        18 &     {\textbf{20$\pm$0.0}} &     {\textbf{20$\pm$0.0}}  &           {\textbf{20$\pm$0.0}}   &        {\textbf{20$\pm$0.0}}  \\
			thoughtful (20)      &    15 &         {\textbf{20}}  &       {\textbf{20}}  &     {\textbf{20$\pm$0.0}}  &    {\textbf{20$\pm$0.0}}  &     {\textbf{20$\pm$0.0}}  &     {\textbf{20$\pm$0.0}} &     {\textbf{20}}  &        {\textbf{20}}   &     {\textbf{20$\pm$0.0}} &    {\textbf{20$\pm$0.0}}  &          {\textbf{20$\pm$0.0}} &        {\textbf{20$\pm$0.0}}  \\
			tidybot (20)         &    17 &        18 &      19 &     {\textbf{20$\pm$0.0}}  &     {\textbf{20$\pm$0.5 }} &     {\textbf{20$\pm$0.0}}  &     {\textbf{20$\pm$0.0}}  &     {\textbf{20}}   &         {\textbf{20}}  &     {\textbf{20$\pm$0.0}} &     {\textbf{20$\pm$0.0}}  &            {\textbf{20$\pm$0.0}}  &       19$\pm$0.5 \\
			tpp (30)                          &    30 &        29 &      29 &     {\textbf{30$\pm$0.6}}  &     {\textbf{30$\pm$0.0}}  &    29$\pm$0.0 &     {\textbf{30$\pm$0.0}} &      {\textbf{30}}  &        {\textbf{30}}   &     {\textbf{30$\pm$0.0}} &     {\textbf{30$\pm$0.0}}  &         {\textbf{30$\pm$0.0}}  &        {\textbf{30$\pm$0.0}} \\
			transport (20)       &    16 &         {\textbf{20}}  &       {\textbf{20}}  &     {\textbf{20$\pm$0.0}}  &   {\textbf{20$\pm$0.0}}   &     {\textbf{20$\pm$0.0 }} &    {\textbf{20$\pm$0.0}}   &     {\textbf{20}}   &         {\textbf{20}}  &  {\textbf{20$\pm$0.0}}    &     {\textbf{20$\pm$0.0}}  &           {\textbf{20$\pm$0.0}} &       {\textbf{20$\pm$0.0}}  \\
			trucks-strips (30)                &     {\textbf{18}}  &        16 &       9 &     9$\pm$0.8 &     9$\pm$1.3 &     9$\pm$0.8 &    10$\pm$1.4 &     11 &         8 &    12$\pm$1.8 &    11$\pm$1.3 &             12$\pm$1.3 &       12$\pm$0.8 \\
			\midrule
			Total (1691)                        &  1456 &      1496 &    1436 &  1455$\pm$8.7 &  1476$\pm$4.2 &  1463$\pm$8.9 &  1502$\pm$4.9 &   1414 &      1456 &  1438$\pm$5.9 &  1462$\pm$8.0 &     {\textbf{1524$\pm$2.5}}  &     1516$\pm$5.0 \\

			\bottomrule
			
		\end{tabular}
		\caption{Coverage over all satisficing benchmarks from IPCs: \textit{complete} - B1: LAMA-\emph{first}, B2: DUAL-BFWS and 'P...', and \textit{polynomial} \textit{incomplete} - B3: $\langle1,2$-C-M$\rangle$ and '$p$-P...'. P$_{\bar \omega}$ refers to BFWS($f_5$) planner with $\cal W$ $=$ $[1,\bar\omega+1]$, ${\cal L}_{\bar \omega}$ is BFWS($f_{\cal L}$) , which uses \textit{Landmarks} , '$I$-' stands for \textit{Iterative}, 'A' for \textit{approximate}, and 'C' for \textit{control} over \textit{open list}.
			The \textit{mean coverage} is shown along with the \textit{standard deviation} for the planners that use sampling. Domains which are fully solved by all planners are omitted but included in supplementary material~\cite{tech-appendix}. The best results are highlighted in bold. } \label{table3}
	}
	
\end{table*}

	At this point, we discuss a new planner, where we iteratively run the  \emph{polynomial}  BFWS($\hat f_5$)  with novelty based pruning, sequentially increasing the number of novelty categories $\cal W$ at each iteration,  $\cal W$ $=$ $[1,\bar\omega+1]$, over $\bar \omega \in [1,|F|]$. We denote the planner as '$pI$-P$_{\bar \omega}$AC' where $I$ stands for \textit{iterative}. Informally, its major advantage is that it taps into the low polynomial space and time complexity of $p$-P$_{\bar \omega}$AC with small $\bar \omega$ values as well as the greater coverage with larger $\bar \omega$ . This can be observed in Table~\ref{table3}, which shows a significant jump in coverage compared to BFWS($f_5$) with novelty based pruning($p$-P$_2$) and $\langle1,2$-C-M$\rangle$ (B3).
	
	The coverage is also higher than the state-of-the-art LAMA-\emph{first} (B1) and DUAL-BFWS (B2). Moreover, from Fig.~\ref{fig6}, we note that '$pI$-P$_{\bar \omega}$AC' planner has better runtime performance than BFWS($f_5$), the winner of Agile track (IPC 2018). It solved $59$ more instances than BFWS($f_5$) across every IPC satisficing benchmarks with a $300$ $sec$ \emph{time} and $8\;GB$ \emph{memory} limit. At the same time. Fig.~\ref{fig3} confirms that the space and time consumption is much less than the baseline BFWS planners.
	 It is worth pointing that '$pI$-P$_{\bar \omega}$AC' is probabilistically \textit{incomplete}. Also, we did not observe any difference in coverage of '$pI$-P$_{\bar \omega}$AC' \emph{with} or \emph{without} the \emph{holding queue}, as the nodes pruned at one iteration get selected in subsequent iterations with positive probability.
	
	\subsubsection{Discussion}
	
	We show that \textit{approximate novelty search} greatly improves the performance over baseline BFWS planners. The ability to compute $\hat w$ $>$ $2$ using novelty approximation, within practical constraints of time and memory, allows us to use '$pI$-P$_{\bar \omega}$AC' configuration that  beats the state-of-the-art. This is impressive for a sequential polynomial planner which uses simple goal counting heuristics $\#g(s)$ and relaxed plan counter $\#r(s)$ along with $\hat w$  to direct the search. Also, we can observe that certain domains were affected more than others. Specifically, the domains \textit{citycar, data-network, hiking and satellite} benefited significantly.

	We found that the \textit{open list control} significantly benefited the domains \textit{citycar and data-network} which have a high branching factor but solvable with $\hat \omega$ $\leq$ $2$. \textit{Citycar} in particular was fully solvable with $\hat \omega$  $=$ $1$  and discarding nodes with $w(s)$ $>$ $1$ didn't impact the order of expansion. \textit{Hiking} and \textit{satellite} on the other hand required expansion of nodes of $w(s)$ $>$ $2$, and the increased coverage highlights the 
	{importance of policy based control of different novelty categories in the open list.}
	 \textit{Childsnack and Floortile} however showed no improvement, which is a combined effect of high width and the fact that our goal count heuristic $\#g(s)$ is not informed enough.

	\section{Conclusion}
	The proposed methods of novelty approximation and open list control in \textit{BFWS} not only have positive impact on coverage but also on the overall time and space complexity of the search, resulting in new state-of-the-art planners over satisficing benchmarks from every IPC since 1998 and more significantly the last 2 IPCs (2014 and 2018). These results strongly suggest that probabilistically complete search algorithms are a promising research direction in classical planning. 
	This is specially crucial in limited time and memory environments where the search must work within hard constraints on time and memory. However, we must note that approximate novelty search is by no means a silver bullet, and certain domains including \textit{Childsnack and Floortile} still remain unsolvable. We hope this work brings about the insights to develop the next generation of
	classical planners, that scale up better as the intractability of the benchmarks ramps up and tackle the inherent limitations of BFS.
	
	\section*{Acknowledgements}
	
	Anubhav Singh is supported by Melbourne Research Scholarship established by the University of Melbourne.
	
	Javier Segovia-Aguas is supported by TAILOR, a project funded by EU H2020 research and innovation programme no. 952215, an ERC Advanced Grant no. 885107, and grant TIN-2015-67959-P from MINECO, Spain.
	
	This research was supported by use of the Nectar Research Cloud, a collaborative Australian research platform supported by the National Collaborative Research Infrastructure Strategy (NCRIS).
	
	\bibstyle{aaai21}
	\bibliography{crossref,biblio}
	
\end{document}